\documentclass{article}

\usepackage{microtype}
\usepackage{graphicx}
\usepackage{subfigure}
\usepackage{booktabs} 

\usepackage{hyperref}

\usepackage[accepted]{icml2025}

\usepackage{amsmath}
\usepackage{amssymb}
\usepackage{mathtools}
\usepackage{amsthm}

\usepackage{adjustbox}
\usepackage{dsfont}
\usepackage{float}
\usepackage{tcolorbox}
\usepackage[capitalize,noabbrev]{cleveref}
\hypersetup{
 colorlinks=true,
 citecolor=cyan,
 linkcolor=orange,
 urlcolor=cyan}

\theoremstyle{plain}
\newtheorem{theorem}{Theorem}[section]
\newtheorem{proposition}[theorem]{Proposition}

\theoremstyle{definition}

\theoremstyle{remark}

\newcommand{\calA}{\mathcal{A}}
\newcommand{\calM}{\mathcal{M}}
\newcommand{\calN}{\mathcal{N}}
\newcommand{\calY}{\mathcal{Y}}
\newcommand{\PP}{\mathbb{P}}
\newcommand{\vx}{\boldsymbol{x}}
\newcommand{\vy}{\boldsymbol{y}}
\newcommand{\ve}{\boldsymbol{e}}
\newcommand{\vz}{\boldsymbol{z}}
\DeclareMathOperator*{\argmax}{arg\,max}
\renewcommand{\algorithmicrequire}{\textbf{Input:}}
\renewcommand{\algorithmicensure}{\textbf{Output:}}
\newcommand{\1}{\mathds{1}}

\usepackage[textsize=tiny]{todonotes}

\icmltitlerunning{Submitted to 2nd Workshop on Test-Time Adaptation for ICML2025}

\begin{document}

\twocolumn[
\icmltitle{Adaptive Diffusion Denoised Smoothing : Certified Robustness via Randomized Smoothing with Differentially Private Guided Denoising Diffusion}

\icmlsetsymbol{equal}{*}

\begin{icmlauthorlist}
\icmlauthor{Frederick Shpilevskiy}{ubc}
\icmlauthor{Saiyue Lyu}{equal,ubc,sn}
\icmlauthor{Krishnamurthy Dj Dvijotham}{sn}
\icmlauthor{Mathias L\'ecuyer}{ubc}
\icmlauthor{Pierre-Andr\'e No\"el}{sn}
\end{icmlauthorlist}

\icmlaffiliation{ubc}{University of British Columbia, Vancouver, Canada}
\icmlaffiliation{sn}{ServiceNow Research}
\icmlcorrespondingauthor{Frederick Shpilevskiy}{fshipil@cs.ubc.ca}
\icmlcorrespondingauthor{Mathias Lécuyer}{mathias.lecuyer@ubc.ca}

\icmlkeywords{Machine Learning, ICML}

\vskip 0.3in ]

\printAffiliationsAndNotice{\icmlEqualContribution} 

\begin{abstract}
We propose Adaptive Diffusion Denoised Smoothing, a method for certifying the predictions of a vision model against adversarial examples, while adapting to the input. Our key insight is to reinterpret a guided denoising diffusion model as a long sequence of adaptive Gaussian Differentially Private (GDP) mechanisms refining a pure noise sample into an image. We show that these adaptive mechanisms can be composed through a GDP privacy filter to analyze the end-to-end robustness of the guided denoising process, yielding a provable certification that extends the adaptive randomized smoothing analysis.
We demonstrate that our design, under a specific guiding strategy, can improve both certified accuracy and standard accuracy on ImageNet for an $\ell_2$ threat model. 
\end{abstract}

\section{Introduction}

Rapid advances in deep learning have enabled models to filter toxic content, assist in patient triage, and steer autonomous vehicles. Despite their remarkable accuracy, these models still remain alarmingly brittle: a few carefully chosen, almost invisible pixel changes can force an image classifier to mislabel the input. Such adversarial attacks have been demonstrated against deep learning systems in medical diagnosis \citep{finlayson2019adversarial}, autonomous driving \citep{eykholt2018robust}, and AI model jailbreaks \citep{carlini2023aligned}, underscoring the critical need for stringent safety and security protection.

Randomized Smoothing (RS) \citep{lecuyer2019certified,cohen2019certified} provides robustness guarantees against adversarial attacks for large models by averaging predictions over noisy versions of the input at test time.
To alleviate the negative impact of noise, recent work on diffusion denoised smoothing (DDS) \citep{carlini2023free, densepure, zhang2023diffsmooth, jeong2023multi} adds one \cite{carlini2023free} or several \cite{densepure} step(s) of diffusion denoising into RS, empirically improving classification performance with the same robustness guarantees.

However, the Gaussian noise required for RS still induces a steep trade-off between robustness and accuracy: to this day, RS can only certify against small attacks, and doing so lowers the accuracy on legitimate instances.
To empirically improve utility,  previous work on adversarial purification \citep{wang2022guided, wu2022guided, bai2024diffusion} has explored injecting guidance during the diffusion denoising process, adaptively steering the denoising trajectories towards a better quality final image. However, these methods do not provide a theoretical analysis of the certification guarantees. While such adaptive techniques can be hard to rigorously analyze \citep{croce2022evaluating,alfarra2022data}, recent work on Adaptive Randomized Smoothing (ARS) \citep{ars} enables test-time adaptivity with rigorous guarantees using GDP composition. Still, this setup does not cover diffusion denoising models, and the authors focus on a specially designed two-step model.

We propose {\bf Adaptive Diffusion Denoised Smoothing (ADDS)}, a more general design for adaptive RS models with a large number of steps based on denoising diffusion models, as illustrated in \cref{fig:pipeline}.
Our key insight is to see guided denoising diffusion models as a long sequence of GDP mechanisms with data-adaptive variance: a pure noise sample is iteratively refined by injecting input-dependent guidance at each step.
We show how to extend ARS with privacy filters for variance adaptive composition, providing an end-to-end certification analysis of guided diffusion denoising for RS. We evaluate our method on ImageNet. 

\begin{figure*}
  \centering
  \includegraphics[width=.95\textwidth]{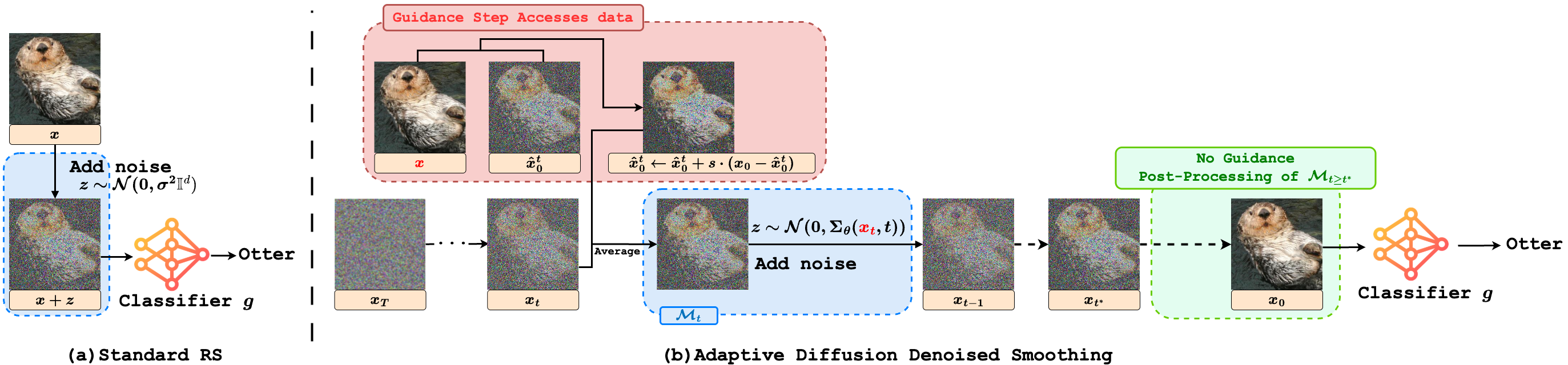}
  \vspace*{-4mm}
  \caption{(a) Randomized Smoothing adds noise to the input to create a smooth model from a base model $g$. (b) Our Adaptive Diffusion Denoised Smoothing starts with a pure noise image $\vx_T$, and guides a reverse-diffusion process at steps $T\leq t\leq t^*$ (the red box, showing one guiding step) towards reconstructing the target image $\vx_0\approx \vx$, for a final prediction by $g$. Starting from $\vx_t$, the pretrained diffusion model predicts a less noisy version of the input $\vx_0^t$, which is then updated by guiding towards $\vx$. The guiding step is then combined with $\vx_t$, leading to a less noisy intermediary image (in the blue box). Finally, the intermediary image is re-noised to output $\vx_{t-1}$.
  We leverage this re-noising step for our DP guarantees at step $t$.
  After timestep $t^*$, there is no guidance thus does not require access to data.  Using GDP composition and privacy filters enable end-to-end analysis over the $T$ steps.}
  \label{fig:pipeline}
\vspace{-0.1cm}
\end{figure*}

\section{Background and Related Work}


\textbf{Adversarial Examples} \citep{adex} of radius $r$ in the $L_p$ threat model are perturbed inputs $\vx+\ve$, with perturbation $\ve\in B_p(r)$ in the $L_p$ ball of radius $r$, that make a classifier $g$ misclassify the input $\vx$, i.e., $g(\vx+\ve)\neq g(\vx)$. These attack inputs are made against classifiers at test time. 

\textbf{Randomized Smoothing (RS)} \citep{lecuyer2019certified,cohen2019certified} is a scalable approach to certify model predictions against any adversarial attacks under $L_2$ norm, which randomizes a base model $g$ by adding spherical Gaussian noise to its input, and produces a smoothed classifier $\calM_s$ that returns the class with highest expectation over the noise: $\calM_S(\vx) \triangleq \arg\max_{\vy \in \mathcal{Y}} \mathbb{P}_{\vz \sim \mathcal{N}(0, \sigma^2 \mathbf{I}^d)}\big(g(\vx + \vz) = \vy\big)$.
Calling $\underline{p_+}, \overline{p_-} \in [0, 1]$ the lower bound on $\calM_s$'s top class prediction and the upper bound on each other prediction, i.e. $\PP(\calM(\vx) = \calM_S(\vx)) \ge \underline{p_+} \ge \overline{p_-} \ge \max_{\vy \neq \calM_S(\vx)} \PP(\calM(\vx) = \boldsymbol{y})$, and with $\Phi^{-1}$ the inverse Gaussian CDF, the certificate size is $r_{\vx} = \frac{\sigma}{2}\big( \Phi^{-1}(\underline{p_+}) - \Phi^{-1}(\overline{p_{-}})  \big)$. That is, $\forall \ve \in B_2(r_{\vx}), \calM_S(\vx + \ve) = \calM_S(\vx)$.

\textbf{Adaptive Randomized Smoothing (ARS)} \citep{ars} extends RS to sequences of data-dependent steps. The method leverages Gaussian Differential Privacy (GDP) \citep{gdp}, an extension of $(\varepsilon, \delta)$-Differential Privacy, to analyze the end-to-end composition of several steps in terms of GDP, with regards to $L_2$-norm adversarial attacks.
Formally, consider $k$ randomized Gaussian mechanisms $\calM_1, \calM_2, \dots, \calM_k$ that apply to an input $\boldsymbol{x}$ and the results of all previous mechanisms, i.e., $m_i \sim \calM_i(\boldsymbol{x} \mid m_{<i})$ for $i = 1,2,\dots,k$. A final classifier applies to the outputs of all steps $g(m_1, m_2, \dots, m_k) =\boldsymbol{y}\in \calY$. Together, these create a composed randomized mechanism $\calM : \boldsymbol{x} \rightarrow g(m_1, m_2, \dots, m_k)$. Define the smoothed classifier $\calM_S$ as $\calM_S(\boldsymbol{x}) \triangleq \argmax_{\boldsymbol{y} \in \calY}\PP(\calM(\boldsymbol{x}) = \boldsymbol{y})$. According to \citet[Theorem 2.3]{ars}, if for all $r \ge 0$ each $\calM_i$ is $\frac{r}{\sigma_i}$-GDP in a $B_2(r)$-neighbourhood, then the smoothed classifier is robust to all perturbations $\boldsymbol{z} \in B_2(r_{\vx})$, such that $\calM_S(\boldsymbol{x}) = \calM_S(\boldsymbol{x} + \boldsymbol{z})$, with
\vspace{-0.3em}
\begin{equation}\label{eq:ars}
    r_{\vx} = \frac{1}{2\sqrt{\sum_{i=1}^k \frac{1}{\sigma^2_i}}}\left(\Phi^{-1}(\underline{p_+}) - \Phi^{-1}(\overline{p_-})\right).
\end{equation}



Notice that in ARS (\cref{eq:ars}) each step $\calM_i$ is $r/\sigma_i$-GDP for an attack of size $r$, with Gaussian noise variance $\sigma^2_i$ fixed in advance. In practice, one might want to adapt the variance based on the results of previous steps.

\textbf{Privacy Filters} \citep{rogers2016privacy} support the composition of privacy mechanisms with adaptive (dependent on previous steps) privacy guarantees, as long as their composition always remains below an upper-bound fixed in advance. Specifically, GDP privacy filters \cite{smith2022fully,koskela2022individual} state that, for a parameter $\mu$ fixed in advance, any composition of $\mu_i$-GDP mechanisms (where $\mu_i$ can be chosen based on the results of previous mechanisms) such that $\sum_i \mu_i^2 \leq \mu^2$ is $\mu$-GDP.


\textbf{Denoising Diffusion Probabilistic Models (DDPM)} \citep{ddpm,ddim} is a type of generative model that learns to denoise a noisy sample. The forward process of DDPM constructs a Markov chain $\{\vx_0,\cdots,\vx_T\}$ from a clean image to pure noise. At each time step, it adds noise to the previous state $\vx_{t-1}$ via $\boldsymbol{x}_t = \sqrt{1-\beta_t}\boldsymbol{x}_{t-1}+\calN(0, \beta_t\mathbf{I})$, where $0<\beta_1<\beta_2<\cdots<\beta_T<1$ control the variances of the diffusion process. Denoting $\bar{\alpha}_t=\prod_{s=1}^t\alpha_s=\prod_{s=1}^t(1-\beta_t)$, we can obtain $q(\vx_t|\vx)=\calN\big(\vx_t; \sqrt{\bar\alpha_t}\vx, (1-\bar\alpha_t)\mathbf{I}\big)$. At each time step of the reverse process, \citet[Equation 11]{ddpm} outputs a reverse Markov chain $p_\theta(\vx_{t-1}|\vx_t)=\calN\big(\vx_{t-1}; \boldsymbol{\mu}_\theta(\vx_t,t),\mathbf{\Sigma}_\theta(\vx_t,t)\big)$, where $\boldsymbol{\mu}_\theta(\vx_t,t)=\frac{1}{\sqrt{\alpha_t}}\big(\boldsymbol{x}_t-\frac{1-\alpha_t}{\sqrt{1 -\bar{\alpha}_t}}\boldsymbol{\epsilon}_\theta(\boldsymbol{x}_t, t)\big)$ and $\boldsymbol{\epsilon}_\theta(\boldsymbol{x}_t, t)$ is a learned predictor. 
Following \citet{nichol2021improved}, models predict a diagonal covariance matrix $\Sigma_\theta(\vx_t,t)$ that depends on $\vx_t$, and where the noise added to the image can differ by pixel. One denoising step follows: 
\vspace{-0.3em}
\begin{equation}\label{eq:ddpm}
    \scalebox{0.95}{$\boldsymbol{x}_{t-1} = \frac{1}{\sqrt{\alpha_t}}\left(\boldsymbol{x}_t-\frac{1-\alpha_t}{\sqrt{1 -\bar{\alpha}_t}}\boldsymbol{\epsilon}_\theta(\boldsymbol{x}_t, t)\right)+\calN(0, \Sigma_\theta(\vx_t,t)).$}
\end{equation}


\textbf{Diffusion Denoised Smoothing (DDS)} \citep{carlini2023free} utilizes off-the-shelf high-fidelity 
diffusion models \citep{dhariwal2021diffusion} as powerful denoisers with no extra training, to remove added smoothing Gaussian noise via a one-shot denoising step. The algorithm first matches the RS perturbed data point $\vx_{\text{rs}}=\vx+\calN(0,\sigma^2\mathbf{I})$ with the noised data point from DDPM forward process $\vx_t=\sqrt{\bar{\alpha}_t}\vx+\calN(0,(1-\bar\alpha_t)\mathbf{I})$ to output a unique time step $t^*$ that satisfies the match. Next it embeds a RS sample on the DDPM trajectory by $\vx_{t^*}=\sqrt{\bar\alpha_{t^*}}\cdot\vx_{\text{rs}}$. Then the one-shot denoising step outputs $\hat{\vx}_0^{t^*}:=\frac{1}{\sqrt{\bar\alpha_{t^*}}}\big(\vx_{t^*}-(\sqrt{1-\bar\alpha_{t^*}})\boldsymbol{\epsilon}_\theta(\boldsymbol{x}_{t^*}, t^*)\big)$ to feed it to the classifier.
DensePure \citep{densepure} achieves better robustness results by repeating a multi-hop denoising process multiple times and taking a majority vote over purified outputs as the final prediction.
%
%

\section{Adaptive Diffusion Denoised Smoothing (ADDS)}
Current DDS certificates do not support adaptive choices during denoising, which leaves potential robustness untapped. To close this gap we propose ADDS, an adaptive version of DDS which demonstrates that privacy-inspired adaptive analysis can tighten certificates for diffusion based defences.
In what follows, we describe a GDP filter (\cref{alg:stop}) based guided DDPM sampling algorithm in (\cref{alg:sampling}).
We first analyze the individual pixel sensitivity for one denoising step (\cref{prop:one}), and use it to compose the analysis over many steps (\cref{prop:sens}). This yields an end-to-end robustness analysis (\cref{prop:adds}).

As noted in Equation (12) of \citet{ddim}, one can generate a sample $\vx_{t-1}$ from a sample $\vx_t$ by utilizing the predicted original image $\hat{\vx}_0^t$. Specifically, we substitute $\hat{\vx}_0^t:=\frac{1}{\sqrt{\bar\alpha_t}}\big(\vx_t-(\sqrt{1-\bar\alpha_t})\boldsymbol{\epsilon}_\theta(\boldsymbol{x}_t, t)\big)$ and rearrange \cref{eq:ddpm} to obtain an expression for the denoising step in terms of the state $\vx_t$ and prediction $\hat{\vx}_0^t$ only,

\vspace{-1.4em}
\begin{align}
\boldsymbol{x}_{t-1}
    &= \tfrac{1}{\sqrt{\alpha_t}}\left(\boldsymbol{x}_t-\tfrac{1-\alpha_t}{\sqrt{1 -\bar{\alpha}_t}}\boldsymbol{\epsilon}_\theta(\boldsymbol{x}_t, t)\right)+\calN(0, \Sigma_\theta(\vx_t,t))\nonumber\\
    &=\tfrac{1}{\sqrt{\alpha_t}}\left(\boldsymbol{x}_t-\tfrac{1-\alpha_t}{\sqrt{1 -\bar{\alpha}_t}}\bigl(\tfrac{1}{\sqrt{1-\bar{\alpha}_t}}\vx_t-\tfrac{\sqrt{\bar{\alpha}_t}}{\sqrt{1-\bar{\alpha}_t}} \hat{\vx}_0^t\bigr) \right)\nonumber\\
    & \,\,\,\,\,\,\,\,\,+\calN(0, \Sigma_\theta(\vx_t,t))\nonumber\\
    &=\tfrac{\sqrt{\bar\alpha_{t-1}}(1-\alpha_t)}{1-\bar{\alpha}_t}\cdot \hat{\vx}_0^t + \tfrac{(1-\bar\alpha_{t-1})\sqrt{\alpha_{t}}}{1-\bar\alpha_t}\cdot \vx_t\nonumber\\
    & \,\,\,\,\,\,\,\,\,+\calN(0, \Sigma_\theta(\vx_t,t))
    \label{eq:ddpmre}
\end{align}
\vspace{-2em}


Inspired by the Backward Universal Guidance of \citet{bansal2023universal}, we introduce guidance in \cref{eq:ddpmre} via shifting $\hat{\vx}_0^t$ towards $\vx$ with scale $s$, yielding $\hat{\vx}_0^t\gets \hat{\vx}_0^t + s\cdot (\vx - \hat{\vx}_0^t)$.
We can thus view the guided $\hat{\vx}_0^t$ as a convex combination $(1-s)\hat{\vx}_0^t+s\vx$, which becomes standard DDPM sampling when $s=0$. 

Using definitions above, we can formulate the guided denoising process as a sequence of GDP mechanisms $\calM_t:\vx\rightarrow \calA_t(\vx)+\vz,\vz\sim \calN(0, \Sigma_\theta(\vx_t,t))$ where $\calA_t(\vx) = \frac{\sqrt{\bar\alpha_{t-1}}(1-\alpha_t)}{1-\bar{\alpha}_t}\left((1-s)\cdot\hat{\vx}_0^t+s\cdot\vx\right)+\frac{(1-\bar\alpha_{t-1})\sqrt{\alpha_{t}}}{1-\bar\alpha_t}\cdot \vx_t$. Note that $\Sigma_{\theta}$ is a diagonal covariance matrix, such that each pixel of $\vx$ gets an independent draw with potentially different variance \citep{nichol2021improved}. We now perform our certified guarantee analysis first on one pixel and then derive the overall guarantee for the whole image based on the pixel-wise results.

Consider a fixed pixel $i$ ($i \in \{1, d\}$) of $\vx$, with an adversarial change of size $r_i$. Denote $\Sigma_{\theta}(\vx_t,t)=\text{diag}(\sigma_{t,1}^2, \cdots,\sigma_{t,i}^2,\cdots,\sigma_{t,d}^2)$.
At step $t$, the denoising process at pixel $i$ is a GDP mechanism $\calM_{t,i}:\vx \rightarrow \calA_{t,i}(\vx)+\vz,\vz\sim\calN(0,\sigma^2_{t,i})$, with sensitivity:
\begin{align*}
    \Delta\calA_{t,i}(\vx)&=\max_{e_i\in B_2(r_i)}\|\calA_{t,i}(\vx+e_i)-\calA_{t,i}(\vx)\|\nonumber\\
    &\leq r_i \cdot s\cdot \frac{\sqrt{\bar\alpha_{t-1}}(1-\alpha_t)}{1-\bar{\alpha}_t} ,
\end{align*}
where $e_j$ is a vector of zeros except for pixel $i$. This directly yields \citep{gdp}:


\begin{algorithm}[tb]
\caption{PrivacyFilter}
\label{alg:stop}
\begin{algorithmic}
    \REQUIRE Per pixel budget $\Lambda, s, t$.
    \STATE $\Lambda' \gets \Lambda - s^2\cdot \frac{\bar\alpha_{t-1}(1-\alpha_t)^2}{(1-\bar{\alpha}_t)^2\cdot \sigma_t^2}$
    \STATE {\bfseries Return:} $\Lambda$, {\em no} {\bfseries if} $\Lambda' \le 0$
    \STATE {\bfseries Return:} $\Lambda'$, {\em ok} {\bfseries otherwise}
\end{algorithmic}
\end{algorithm}

\begin{algorithm}[tb]
\caption{Clean Image Guided Denoising}
\label{alg:sampling}
\begin{algorithmic}
   \REQUIRE original image $\vx$, guidance scale $s$, total RS variance $\sigma^2$
   \STATE Initialize $\vx_T\gets$ sample from $\mathcal{N}(0,\textbf{I})$
   \STATE Initialize $\mu = 1/\sigma$, $\Lambda \gets \mu \1^d$  \hfill{\footnotesize $	\triangleright$ \texttt{Vector of all $\mu$}}
   \FOR{$t$ {\bfseries from} $T$ {\bfseries to} $1$}
   \STATE $\hat{\vx}_0^t\gets\frac{1}{\sqrt{\bar\alpha_t}}\big(\vx_t-(\sqrt{1-\bar\alpha_t})\boldsymbol{\epsilon}_\theta(\boldsymbol{x}_t, t)\big)$
   \STATE $\Lambda$, filter $\gets$ PrivacyFilter($\Lambda, s, t$) \hfill {\footnotesize $\triangleright$ \texttt{Alg.\ref{alg:stop} filter}}
   \IF{filter == {\em ok}}
   \STATE $\hat{\vx}_0^t\gets \hat{\vx}_0^t + s\cdot (\vx - \hat{\vx}_0^t)$
   \ENDIF
   \STATE $\vx_{t-1}\sim \mathcal{N}(\frac{\sqrt{\bar\alpha_{t-1}}(1-\alpha_t)}{1-\bar{\alpha}_t}\cdot \hat{\vx}_0^t + \frac{(1-\bar\alpha_{t-1})\sqrt{\alpha_{t}}}{1-\bar\alpha_t}\cdot \vx_t, \Sigma_{\theta}(\vx_t,t))$
   \ENDFOR
   \ENSURE $\vx_0$
\end{algorithmic}
\end{algorithm}

\begin{proposition}[One step denoising budget]\label{prop:one}
    \begin{align*}\label{eq:per-pixel-step}
        \mu_{t,i}^2=\tfrac{\Delta\calA_{t,i}^2}{\sigma_{t,i}^2} = \frac{r^2_i}{\sigma_{t,i}^2} \cdot s^2 \cdot \frac{\bar\alpha_{t-1}(1-\alpha_t)^2}{(1-\bar{\alpha}_t)^2} .
    \end{align*}
\end{proposition}


\begin{proposition}[End-to-end Pixel GDP]\label{prop:sens}
    Consider as neighbours any two inputs differing in pixel $i$, $\vx$ and $\vx + e_i$, by a size at most $r_i$ ($e_i\in B_2(r_i)$). Under this neighbouring definition, Algorithm \ref{alg:sampling} is $\tfrac{r_i}{\sigma}$-GDP.
\end{proposition}
\begin{proof}
Algorithm \ref{alg:stop} ensures that over any run of Algorithm \ref{alg:sampling} we have $\sum_t \frac{1}{\sigma_i^2} \cdot s^2 \cdot \frac{\bar\alpha_{t-1}(1-\alpha_t)^2}{(1-\bar{\alpha}_t)^2} \leq \mu^2$. This in turns implies that $\sum_t \mu^2_{t,i} \leq r_i^2 \mu^2 = \big( \tfrac{r_i}{\sigma} \big)^2$. This is a valid GDP filter \cite{smith2022fully,koskela2022individual}, ensuring that Algorithm \ref{alg:sampling} is $\tfrac{r_i}{\sigma}$-GDP.
\end{proof}

Based on the pixel wise results we can obtain the certified guarantee for the whole image:

\begin{theorem}[Adaptive Diffusion Denoised Smoothing]\label{prop:adds}
    Consider the guided DDPM denoising process from Algorithm \ref{alg:sampling}, with total Randomized Smoothing variance $\sigma^2$, coupled with a predictive model $g(\cdot)$. Consider the associated smoothed model $M_S: \vx \rightarrow \argmax_{y\in\mathcal{Y}}\PP_{\vx_0 \sim Alg. \ref{alg:sampling}}(g(\vx_0)=y)$.
    
    Let $y_+ \triangleq M_S(\vx)$ be the prediction on input $\vx$, and let $\underline{p_+}, \overline{p_-} \in [0, 1]$ be such that $\PP(g(\vx_0) = y_+) \geq \underline{p_+} \geq \overline{p_-} \geq \max_{y_- \neq y_+} \PP(g(\vx_0) = y_-)$.
    
    Then $\forall e \in B_2(r_x), \ M_S(\vx + e) = M_S(\vx)$, with:
    \begin{align*}
        r_{\vx}=\frac{\sigma}{2}\left( \Phi^{-1}(\underline{p_+}) - \Phi^{-1}(\overline{p_-}) \right) .
    \end{align*} 
\end{theorem}
\begin{proof}
Consider any adversarial change $e \in B_2(r_x)$. Denoting $r_i$ the change at each pixel $i$, we have that $\sum_i r_i^2 = r^2_x$. By Proposition \ref{prop:sens}, we also know that for any $i$, considering only the pixel $i$ mechanism is $\tfrac{r_i}{\sigma}$-GDP. By concurrent composition \cite{haney2023concurrent} of the GDP mechanisms over each pixel, Algorithm \ref{alg:sampling} is $\tfrac{\sqrt{\sum_i r^2_i}}{\sigma} = \tfrac{r_x}{\sigma}$-GDP. Applying Corollary 2.2 of \citet{ars} concludes the proof.
\end{proof} 
\section{Experiments}

We evaluate certified $\ell_2$ robustness on ImageNet \citep{deng2009imagenet}. We follow \citet{carlini2023free}, using the unconditional 256 $\times$ 256 guided diffusion model of \citet{dhariwal2021diffusion} and a pre-trained BEiT large model \citep{bao2021beit} as classifier ($88.6\%$ top-1 validation accuracy). We use three noise levels $\sigma\in\{1.0, 1.5, 2.0\}$ and randomly select 250 samples (each of a different class) from the ImageNet validation set for certification. We denoise sequentially over 20 evenly-spaced timesteps from the original 1000 denoising steps (i.e. 999, 949, 899, \dots). We compare the certified accuracy at $r=0$ and the clean accuracy of several methods: ADDS with 1 and 5 votes, \citet{carlini2023free}, and DensePure \citep{densepure} with 1 and 5 votes.
We also evaluate ADDS without doing unguided denoising past $t^*$, where we perform one-shot sampling of the final image (like in \citet{carlini2023free}) once out of budget ($\Lambda=0$ in Algorithm \ref{alg:sampling}).



\begin{table}[t]
\vspace{-0.1in}
\caption{Certified accuracy (at $r=0$) for ImageNet}
\label{tab:imagenet}
\begin{center}
\begin{small}
\begin{tabular}{c|ccc}
\toprule
$\sigma$ & 1.0 & 1.5 &  2.0\\
\midrule
Carlini et al.  & \textbf{62.0} & 38.4 & 26.8 \\
DensePure &  57.6 & 40.0 & 25.6 \\
DensePure w/ 5 votes & 61.6 & 45.6 & 31.2 \\
{\bf ADDS (Ours)} & 58.8  & 40.8 & 27.6 \\
{\bf ADDS w/ 5 votes} & 60.4  & \textbf{46.8} & \textbf{32.0} \\
{\bf ADDS (w/o unguided denoising)} & 61.2  & 44.8 & 31.2 \\
\bottomrule
\end{tabular}
\end{small}
\end{center}
\vspace{-0.4cm}
\end{table}

\begin{table}[t]
\vspace{-0.1in}
\caption{Clean accuracy for ImageNet}
\label{tab:imagenet-clean}
\begin{center}
\begin{small}
\begin{tabular}{c|ccc}
\toprule
$\sigma$ & 1.0 & 1.5 &  2.0\\
\midrule
Carlini et al.  & 69.6 & 55.2 & 46.8 \\
DensePure &  68.4 & 58.0 & 46.4\\
DensePure w/ 5 votes & 68.4 & 55.2 & 45.2\\
{\bf ADDS (Ours)} & 68.8 & 58.0 & 47.6 \\
{\bf ADDS w/ 5 votes} & 68.8  & 57.2 & 46.8 \\
{\bf ADDS (w/o unguided denoising)} & \textbf{70.0}  & {\bf 60.0} & {\bf 48.0} \\
\bottomrule
\end{tabular}
\end{small}
\end{center}
\vspace{-0.5cm}
\end{table}

\cref{tab:imagenet} shows the certified accuracy at $r=0$, and \cref{tab:imagenet-clean}, the clean accuracy.
ADDS without unguided denoising performs best in clean accuracy.
In certified accuracy, \citet{carlini2023free} performs best at $\sigma=1.0$, and ADDS with 5 votes is best at larger noise. ADDS without unguided denoising is competitive across all noise levels. We make three important observations that explain these results.

First, unguided denoising (in DensePure and ADDS after the budget is exhausted) increases variance. Indeed, conditioned on the robustness noise draws (i.e., DensePure's diffusion noising process or the noise during ADDS guiding), additional denoising steps introduce more noise, and hence variance.
We can see on \cref{tab:imagenet-clean} that this variance systematically degrades clean accuracy. As a result, \citet{carlini2023free} outperforms in low noise ($\sigma=1$) and ADDS without unguided denoising is always best.

Second, and perhaps surprisingly, going from 1 vote to 5 votes, which alleviates the variance increase, degrades the clean accuracy of both DensePure and ADDS, while increasing certified accuracy.
This is because voting concentrates predictions on the top class, conditioned on robustness noise draws.
On ``easy'' images that are often well classified, this increases the top (correct) probability, thereby improving certified accuracy.
However, on harder images where several labels typically have high probability, majority voting can bias predictions toward the incorrect label, reducing pure accuracy (see details in \cref{appendix:exp-details} and \cref{fig:adds_voting_distribution}).

Third, compared to \citet{carlini2023free}, ADDS makes a trade-off between original image fidelity and noise: each guided denoising step mixes the predicted image $\hat{\vx}_0^t$ with $\vx$ such that the signal component of $\vx_t$ is not composed entirely of the original image, but also of predictions from the denoising model.
As a result, when ADDS runs out of budget, it reaches a state $\vx_{t^*}$ with less noise that the starting point of \citet{carlini2023free} (which uses all of the budget to sample $\vx_t$ from $q(\vx_t|\vx)$).
This translates to an increased level of detail in generated images, in exchange for the potential of deviating from the original (e.g. a blurry dog might become a detailed cat).
The tradeoff is more pronounced at higher noise levels ($\sigma=1.5, 2.0$), at which ADDS without unguided denoising also outperforms \citet{carlini2023free} in robust accuracy. At lower noise ($\sigma=1.0$) the noise scale sampled by \citet{carlini2023free} is already small enough for good one-shot denoising in most cases (see details in \cref{appendix:exp-details} and \cref{fig:images_sigma=1.0,fig:images_sigma=1.5,fig:images_sigma=2.0}).



\clearpage
\section*{Acknowledgements}
FS was partially supported by a UBC Four Year Fellowship and SL was partially supported by ServiceNow and Mitacs. We are grateful for the financial support of the Natural Sciences and Engineering Research Council of Canada (NSERC) [reference number RGPIN-2022-04469]. This research was enabled by computational support provided by the Digital Research Alliance of Canada (alliancecan.ca), by the University of British Columbia’s Advanced Research Computing (UBC ARC), and by ServiceNow.
\bibliography{main}
\bibliographystyle{icml2025}

\newpage
\appendix
\onecolumn

\section{Extended Related Work and Background}

\textbf{Gaussian Differential Privacy (GDP or $f$-DP)} \citep{gdp} is an extension of $(\varepsilon, \delta)$-Differential Privacy that defines privacy according to the power of any hypothesis test to differentiate a Gaussian-distributed output from any of its neighbours. Consider a Gaussian mechanism $\calM : \boldsymbol{x} \rightarrow \calA(\boldsymbol{x}) + \boldsymbol{z}$, where $\boldsymbol{z} \sim \calN(0,\frac{r^2}{\mu^2}\mathbf{I})$ and $\calA$ is some model. According to \citet[Theorem 2.7]{gdp}, for any neighbouring inputs $\boldsymbol{x}, \boldsymbol{x'}$ such that $\|\calA(\boldsymbol{x}) - \calA(\boldsymbol{x'})\|_2^2 \le r$, $\calM$ is $\mu$-GDP.

\textbf{DDS Variations}
Concurrently to \citet{densepure}, DiffSmooth \citep{zhang2023diffsmooth} proposes a one-shot denoising procedure that locally smooths the output by taking a majority vote of the predictions for each purified sample over multiple Gaussian noise samples. This process relies on much stronger assumptions than the RS approach we, \citet{densepure}, and \citet{ars} build on, which we believe is risky in the adversarial setting we consider.

\textbf{Adding Guidance during Denoising}. Recent work \citep{dhariwal2021diffusion,bansal2023universal} has explored different techniques of injecting guidance into diffusion models to steer the denoising trajectories towards task-specific semantics, such that controllable high-fidelity outputs are sampled without retraining the underlying generator. To further improve robustness, various methods \citep{wang2022guided,wu2022guided,bai2024diffusion} has utilized guidance with diffusion in adversarial purification. Although obtaining empirical accuracy improvement, these methods do not provide provable certification. Inspired by \cite{bansal2023universal}, our adaptive pipeline takes advantage of the guidance during denoising, while the certification analysis stays sound with the benefit from ARS.

\section{Extended Experiments}
\label{appendix:exp-details}

\textbf{Implementation Details}. We did all our experiments on a 250 image subset of the 50,000 image ImageNet \citep{deng2009imagenet} validation set, sampling every 200th image. We computed the top class lower bound over 1000 samples and the certified radius of 10,000 samples. For ADDS (every version), we chose guiding scales of 0.8 at $\sigma=1.0, 1.5$ and 0.9 at $\sigma=2.0$ and fixed 20 denoising steps (i.e. 999, 949, 899, \dots). We tuned these parameters in the $\sigma=1.0,1.5$ settings on ADDS (1 vote). If we are unable to do a guiding step at the guiding scale we chose, we compute the largest possible guiding scale such that all of the budget is expended. For \citet{carlini2023free}, we found that noising the image by $\sigma$ performed better than by the noise scale corresponding to the timestep that is at least as noisy as $\sigma$ (i.e. for timestep $t$, the noise scale would be $\sqrt{1-\bar{\alpha}_t} \ge \sigma$) -- DensePure also noises the image by $\sigma$ rather than $\sqrt{1-\bar{\alpha}_t}$.

\begin{figure*}[ht]
\vspace{-0.2cm}
\centering
{\includegraphics[width=\textwidth]{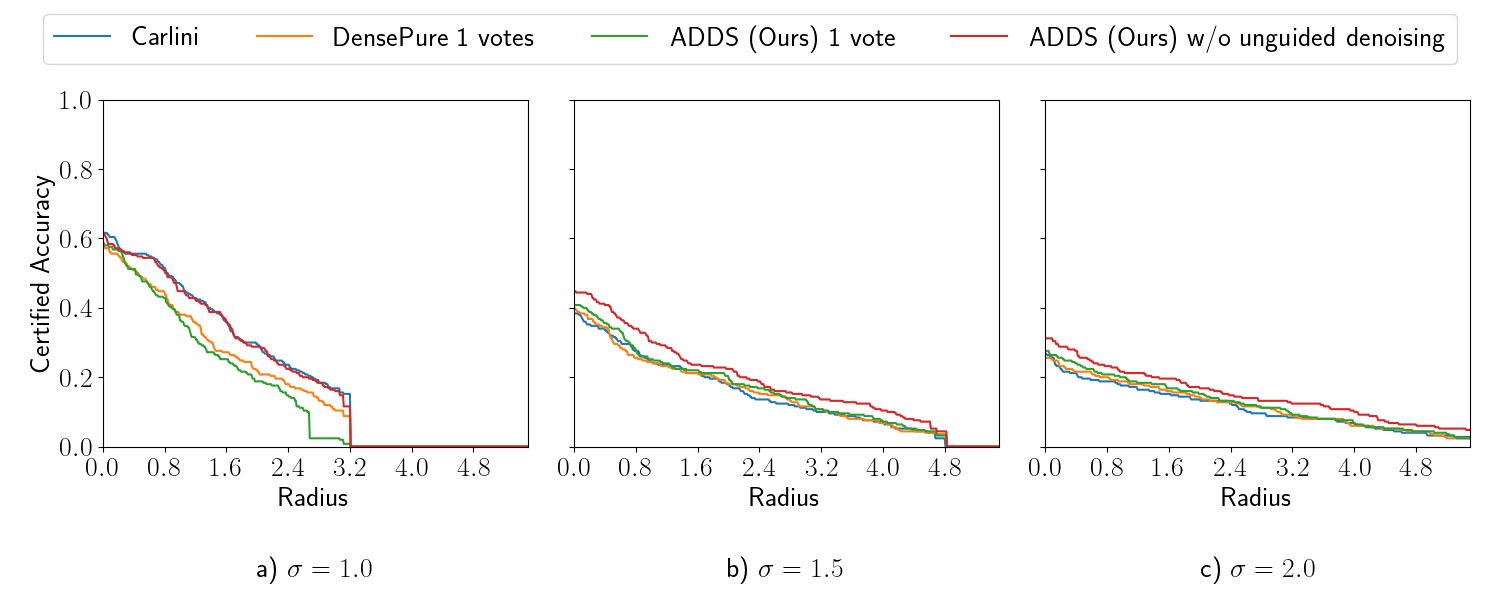}}
\vspace{-0.8cm}
\caption{\footnotesize%
\textbf{Certified Test Accuracy of 1 Vote Methods on ImageNet.}
Certified accuracy of ADDS with 1 vote is competitive with \citet{densepure} in (a) and with \citet{carlini2023free} in (b) and (c). Although at radius 0, we showed in \cref{tab:imagenet} that we improve on \citet{carlini2023free} by two percentage points, (b) and (c) show that for larger radius, the gap tightens. ADDS without unguided de-noising is competitive with \citet{carlini2023free} in (a) and performs better than both in (b) and (c).}
\label{fig:vote_1_certification}
\end{figure*}

In the non-voting setting, \cref{fig:vote_1_certification} shows that ADDS without unguided denoising performs the best at $\sigma=1.5, 2.0$ at both radius 0 and larger radii. At $\sigma=1.0$, \citet{carlini2023free} performs the best and is matched by ADDS without unguided denoising. The superior performance of these methods at $\sigma=1.0$ is because they do not do unguided denoising. At the timestep in the denoising process where the noise has scale $\sigma=1.0$, the predictions of the denoiser are clear and unambiguous. As a result, the classifier is able to accurately predict on the denoised images. This is reflected by the high pure accuracy of \citet{carlini2023free} and ADDS without unguided denoising in \cref{tab:imagenet-clean}. And since these methods do not do unguided denoising -- which introduces an additional source of variance -- they have less variance in classifications than DensePure and ADDS. As a result, they are able to certify more correctly predicted images, and consequently obtain the highest certified accuracies. The reason why \citet{carlini2023free} performs better than ADDS without unguided denoising is because ADDS still introduces some additional variance because the guiding scale $s < 1$.

\begin{figure*}[ht]
\vspace{-0.2cm}
\centering
{\includegraphics[width=\textwidth]{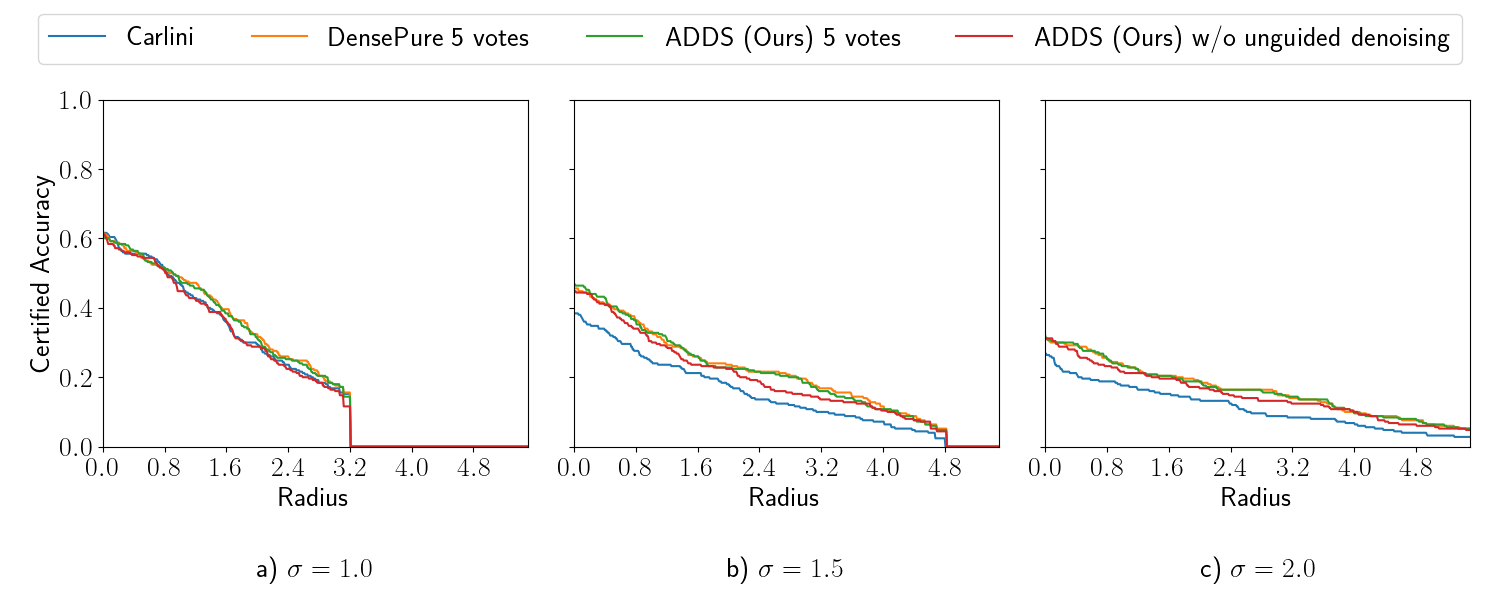}}
\vspace{-0.8cm}
\caption{\footnotesize%
\textbf{Certified Test Accuracy of 5 Vote Methods on ImageNet.}
Certified accuracy of ADDS with 5 vote is competitive with \citet{densepure} with 5 votes at every $\sigma$. ADDS without unguided de-noising is outperformed by \citet{densepure} and ADDS with 5 votes, but the gap tightens as $\sigma$ increases.}
\label{fig:vote_5_certification}
\end{figure*}

In \cref{fig:vote_5_certification}, we see that majority voting over 5 unguided trajectories increases the certified accuracy of DensePure and ADDS. Both DensePure and ADDS with 5 votes match and slightly exceed the certified accuracy of \citet{carlini2023free} and ADDS without unguided denoising. In \cref{fig:adds_voting_distribution}, we see the effect of majority voting on a selection of images where voting causes a change in classification or certification. For example, at $\sigma=1.5$, image 12400 changes from abstained to certified and image 23200 changes from misclassified to correctly classified (although still abstained). It should be noted that \cref{fig:adds_voting_distribution} does not reflect that the vast majority of these changes are abstains that become certified. Voting concentrates counts on the most frequent classes, conditioned on robustness noise draws. The count of each highest frequency class increases -- shown as the class becoming more yellow. By further concentrating classification towards the highest frequency classes, voting is able to increase the certified radius -- reflected in \cref{tab:imagenet} by an improvement in certified accuracy. We can see that all misclassifications as a result of voting occur in abstains (e.g. images 18200, 46000, 47000), where multiple classes have high counts. In these cases, certified accuracy stays the same, but pure accuracy reduces. \cref{fig:runtime} shows that voting comes with the drawback of being more computationally expensive.

\begin{figure*}[ht]
\vspace{-0.2cm}
\centering
{\includegraphics[width=0.8\textwidth]{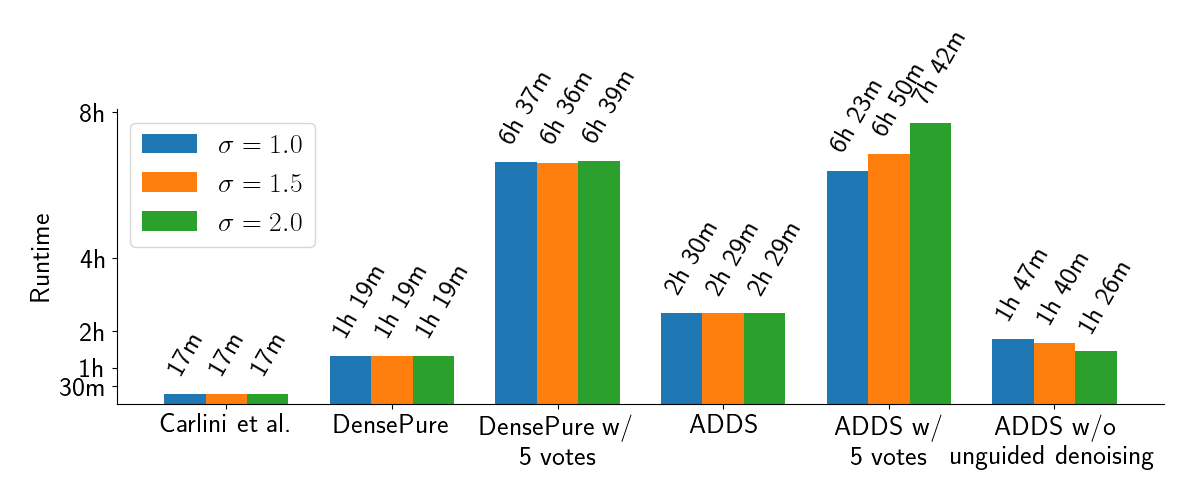}}
\vspace{-0.7cm}
\caption{\footnotesize%
\textbf{Runtime Plot.}
The time it takes each algorithm from \cref{tab:imagenet} and \cref{tab:imagenet-clean} to certify one image over 10,000 certification samples. We run each algorithm on one A100-40GB GPU on the Narval Compute Canada cluster.}
\label{fig:runtime}
\end{figure*}

\begin{figure*}[ht]
\centering
{\includegraphics[width=\textwidth]{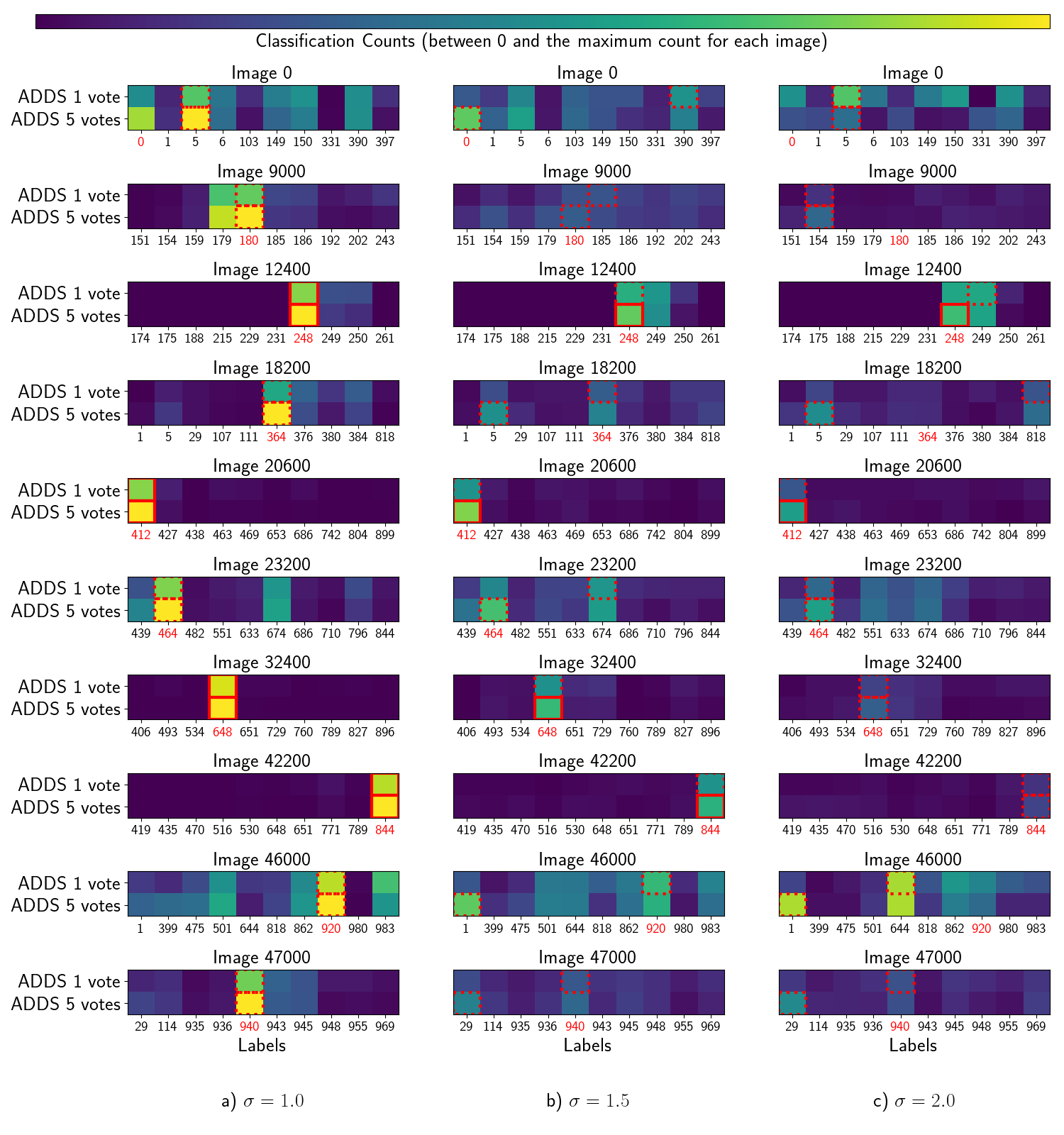}}
\vspace{-0.7cm}
\caption{\footnotesize%
\textbf{Distribution of Classification Counts of ADDS over 1 and 5 Votes.}
The plot shows the classification counts of ADDS at 1 and 5 votes over 1000 noise seeds on a selection of 10 images. Each column corresponds to a different privacy budget. For each image, we show the 10 most frequent labels, with the label in red being the true label. The solid red box corresponds to a certified prediction and the dashed red box corresponds to a prediction that is not certified. The selected images were chosen because they exhibited a change in classification correctness or certification as a result of voting. The proportion of these changes is not reflected. In particular, the number of occurrences where voting causes an abstained classification to become certified (for example, image 42200 at $\sigma=1.5$) far exceeds the number of correct classifications that become incorrect (image 18200 at $\sigma=1.5$).}
\label{fig:adds_voting_distribution}
\end{figure*}

\begin{figure*}[ht]
\centering
{\includegraphics[width=\textwidth]{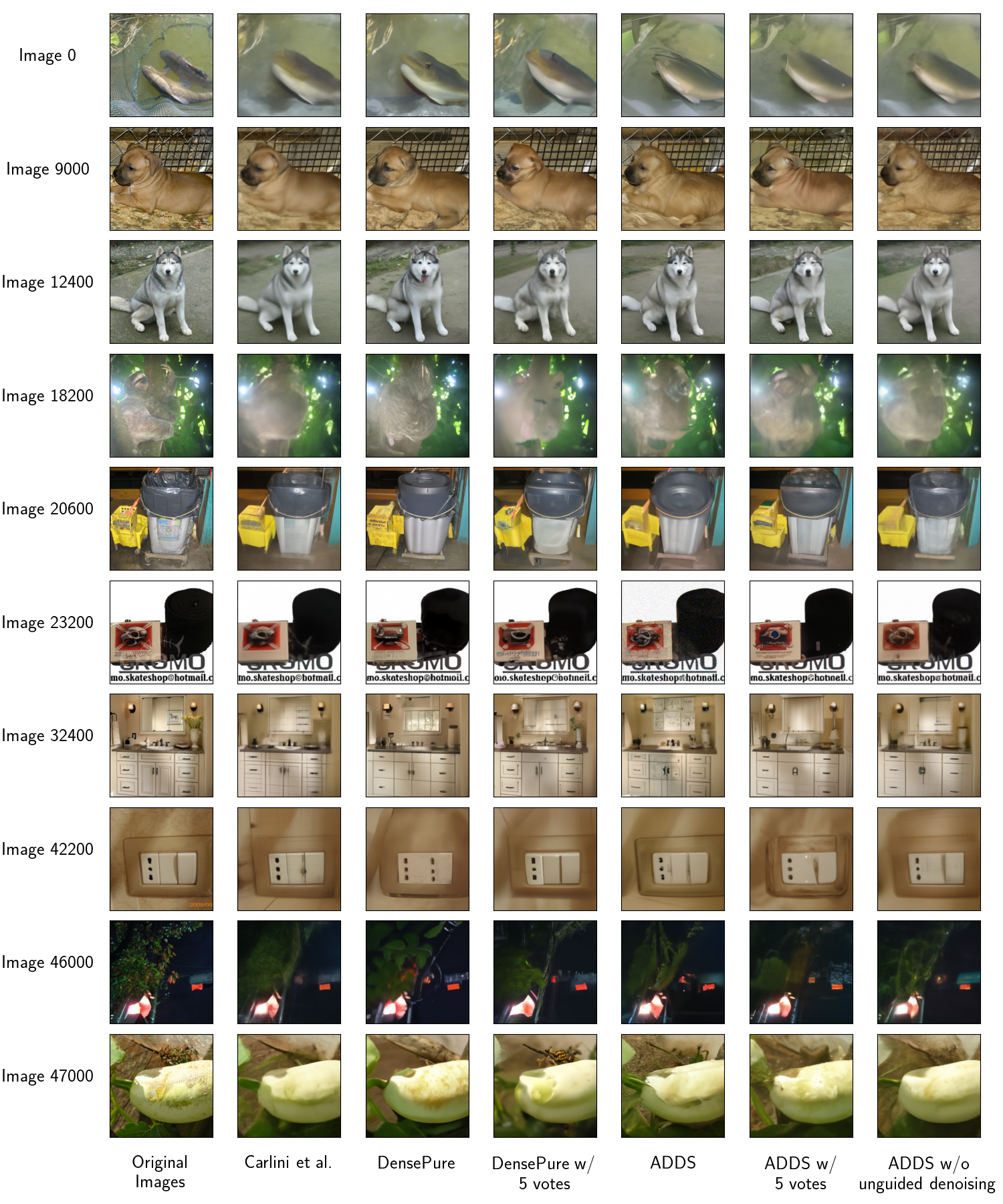}}
\vspace{-0.7cm}
\caption{\footnotesize%
\textbf{Selection of Denoised Images at $\boldsymbol{\sigma=1.0}$.}}
\label{fig:images_sigma=1.0}
\end{figure*}

\begin{figure*}[ht]
\centering
{\includegraphics[width=\textwidth]{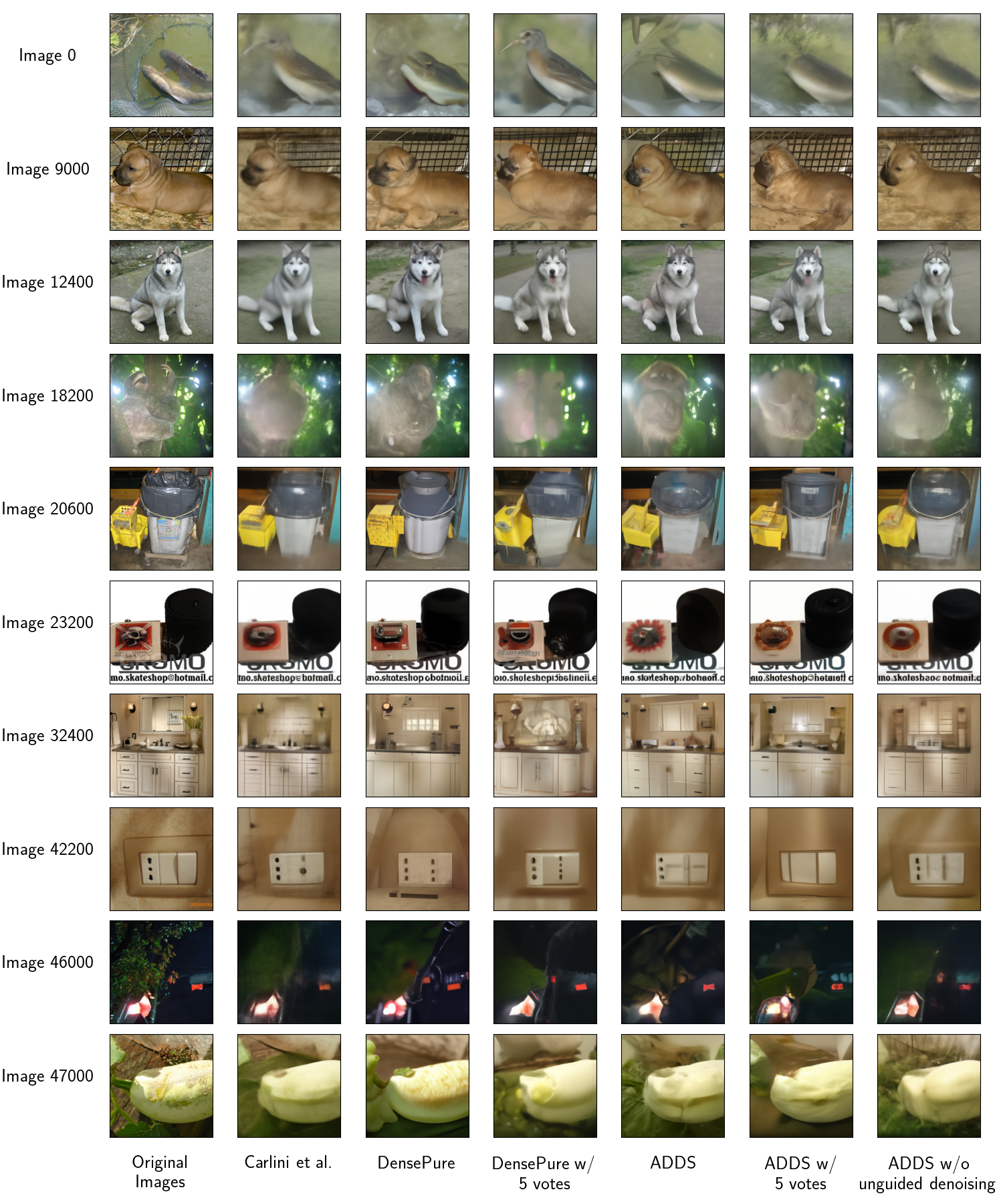}}
\vspace{-0.7cm}
\caption{\footnotesize%
\textbf{Selection of Denoised Images at $\boldsymbol{\sigma=1.5}$.}}
\label{fig:images_sigma=1.5}
\end{figure*}

\begin{figure*}[ht]
\centering
{\includegraphics[width=\textwidth]{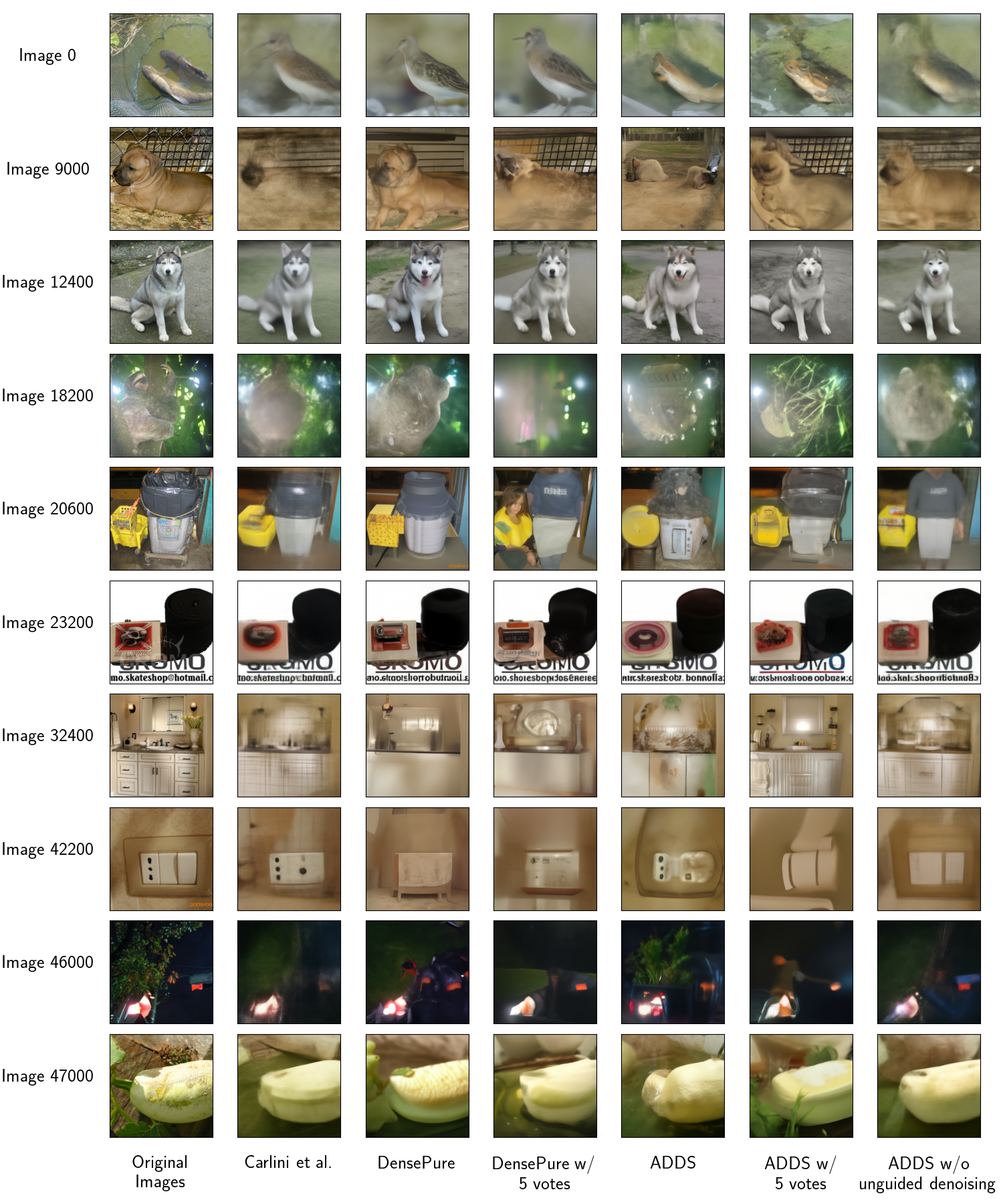}}
\vspace{-0.7cm}
\caption{\footnotesize%
\textbf{Selection of Denoised Images at $\boldsymbol{\sigma=2.0}$.}}
\label{fig:images_sigma=2.0}
\end{figure*}


\end{document}